\documentclass{article}

\usepackage[nonatbib, preprint]{nips_2018}

\usepackage{times}
\usepackage{cite}
\usepackage{subfigure}
\usepackage{graphicx}
\usepackage{amsmath}
\usepackage{bm}
\usepackage{url}
\usepackage{stmaryrd}
\usepackage{balance}
\usepackage{amssymb}
\usepackage{pgfplots}
\usepackage{pifont}
\usepackage[usenames,dvipsnames]{pstricks}
\usepackage{epsfig}
\usepackage{booktabs}
\usepackage{pgffor}
\usepackage{tikz}
\usepackage{multirow}

\newcommand{\mb}{\mathbf}

\newtheorem{theo}{\textsc{Theorem}}
\newtheorem{lemma}{\textsc{Lemma}}
\newtheorem{proof}{\textsc{Proof}}
\newtheorem{defn}{\textsc{Definition}}

\usepackage{algorithm}
\usepackage{algorithmic}

\title{On Deep Ensemble Learning from a Function Approximation Perspective}

\author{Jiawei~Zhang$^\star$, Limeng Cui$^\dagger$, Fisher B. Gouza$^\dagger$\\
$^\star$IFM Lab, Florida State University, FL, USA\\
$^\dagger$University of Chinese Academy of Sciences, Beijing, China\\
jzhang@cs.fsu.edu, lmcui932@163.com, fisherbgouza@gmail.com}

\begin{document}

\maketitle

\vspace{-15pt}
\begin{abstract}
\vspace{-10pt}

In this paper, we propose to provide a general ensemble learning framework based on deep learning models. Given a group of unit models, the proposed \textit{deep ensemble learning} framework will effectively combine their learning results via a multi-layered ensemble model. In the case when the unit model mathematical mappings are \textit{bounded}, \textit{sigmoidal} and \textit{discriminatory}, we demonstrate that the \textit{deep ensemble learning} framework can achieve a universal approximation of any functions from the input space to the output space. Meanwhile, to achieve such a performance, the \textit{deep ensemble learning} framework also impose a strict constraint on the number of involved unit models. According to the theoretic proof provided in this paper, given the input feature space of dimension $d$, the required unit model number will be $2^d$, if the ensemble model involves one single layer. Furthermore, as the ensemble component goes deeper, the number of required unit model is proved to be lowered down exponentially.

\end{abstract}
\vspace{-10pt}
\vspace{-10pt}
\section{Introduction}\label{sec:intro}
\vspace{-10pt}

In recent years, deep learning has achieved a remarkable success in various areas, including \textit{computer vision} \cite{HS97, KSH12, LBH15, KTSLSL14}, \textit{natural language processing} \cite{ASKR12, MH09, DHK13, HDYDMJSVNSK12}, and \textit{network embedding} \cite{WCZ16, CHTQAH15, PAS14}. Meanwhile, by this context so far, the deep learning research is normally narrowed to the scope of deep neural networks, and it lacks theoretic foundations as well as the interpretability of the learning process, which greatly hinders its applications in many areas. 

In this paper, we propose a novel learning model by combine deep learning with ensemble learning, which is formally named as the \textit{deep ensemble learning} model. Given a group of unit models (of different shape and categories), the proposed \textit{deep ensemble learning} model will effectively combine their learning results together via a deep ensemble architecture. The \textit{deep ensemble learning} model propose in this paper is a general framework, and various machine learning models can be incorporated as the uint model, including both deep models and traditional non-deep models. 


Formally, let function $g(\cdot; \mb{\theta})$ denote the true (unknown) mapping from the feature space $\mathcal{X}$ to the label space $\mathcal{Y}$, where $\mb{\theta}$ denotes the function parameters. To learn such a mapping, let set $\mathcal{F} = \{f^1(\cdot; \mb{w}^1), f^2(\cdot; \mb{w}^2), \cdots, f^n(\cdot; \mb{w}^n)\}$ denote the $n$ unit models to be incorporated in the \textit{deep ensemble learning} model, which are parameterized by vectors $\mb{w}^1$, $\mb{w}^2$, $\cdots$, and $\mb{w}^n$ respectively. Unit models in $\mathcal{F}$ can all project data instance from the feature space $\mathcal{X}$ to the label space $\mathcal{Y}$. Given a data instance featured by vector $\mb{x} \in \mathcal{X}$, we can represent its true label as $\mb{y} = g(\mb{x}; \mb{\theta})$. Meanwhile, for the same data instance (i.e., $\mb{x}$), its prediction labels by the \textit{deep ensemble learning} model can be denoted as $\hat{\mb{y}} = F(\mb{x}) = h(\hat{\mb{y}}^1, \hat{\mb{y}}^2, \cdots, \hat{\mb{y}}^n)$, where $h(\cdot)$ is the \textit{ensemble function} and $\hat{\mb{y}}^1 = f^1(\mb{x}; \mb{w}^1)$, $\hat{\mb{y}}^2 = f^2(\mb{x}; \mb{w}^2)$, $\cdots$, and $\hat{\mb{y}}^n =f^n(\mb{x}; \mb{w}^n)$ denote the prediction labels by the unit models respectively.

To effectively integrated the unit models, there should exist some strategies to combine their prediction results together in the \textit{deep ensemble learning} model. Here, if we assume the \textit{deep ensemble learning} model assignes a unit weight for the prediction labels by different unit models, e.g., $v^i$ for the $i_{th}$ unit model $f^i(\cdot; \mb{w}^i)$, we can formally represent the prediction results of the \textit{deep ensemble learning} model as follows:\vspace{-5pt}
\begin{equation}
\hat{\mb{y}} = F(\mb{x}) = h(\mb{z}; \mb{x}) \mbox{, where } \mb{z} =  \sum_{i = 1}^n v^i \cdot \hat{\mb{y}}^i = \sum_{i = 1}^n v_i \cdot f^i(\mb{x}; \mb{w}^i).
\end{equation}

\vspace{-8pt}
\noindent According to the equation, performance of the \textit{deep ensemble learning} model is highly dependent on $4$ key factors:
\begin{itemize} 
\vspace{-8pt}
\item unit learning models, e.g., $f^i(\cdot; \mb{w}^i) \in \mathcal{F}$,
\vspace{-3pt}
\item unit model weights, e.g., $v^i \in \mathbb{R}$, 
\vspace{-3pt}
\item unit model number $n \in \mathbb{N}^+$, 
\vspace{-3pt}
\item the {deep ensemble function} $h(\cdot)$.
\vspace{-8pt}
\end{itemize} 

In this paper, we will analyze the \textit{deep ensemble learning} model from these $4$ aspects in great detail. The following part of this paper will be organized as follows. In Section~\ref{sec:relatedwork}, we will draw a summary about the existing related works of deep learning and ensemble learning. Notations and terminologies used in the paper will be introduced in Section~\ref{sec:definition}. The unit model properties, ensemble weights and the ensemble function will be analyzed in Section~\ref{subsec:analysis}. Under different deep ensemble functions and settings, we will study the required unit model number $n$ in Section~\ref{sec:analysis}. Extensive numerical experiments of the \textit{deep ensemble learning} model on several benchmark datasets will be illustrated in Section~\ref{sec:experiment}. Finally, we will conclude this paper in Section~\ref{sec:conclusion}.
\vspace{-10pt}
\section{Related Works} \label{sec:relatedwork}
\vspace{-10pt}

This paper is closely related to two research topics, including \textit{deep learning} and \textit{ensemble learning}.

\noindent \textbf{Deep Learning}: The essence of deep learning is to compute hierarchical features or representations of the observational data \cite{GBC16, LBH15}. With the surge of deep learning research and applications in recent years, lots of research works have appeared to apply the deep learning methods, like deep belief network \cite{HOT06}, deep Boltzmann machine \cite{SH09}, deep neural network \cite{J02, KSH12} and deep autoencoder model \cite{VLLBM10}, in various applications, like speech and audio processing \cite{DHK13, HDYDMJSVNSK12}, language modeling and processing \cite{ASKR12, MH09}, information retrieval \cite{H12, SH09}, objective recognition and computer vision \cite{LBH15}, as well as multimodal and multi-task learning \cite{WBU10, WBU11}. Traditional deep learning models have too many parameters, Zhou et al. introduce the deep forest as an alternative approach in \cite{ijcai2017-497} to combine the decision trees with a deep architecture.

\noindent \textbf{Ensemble Learning}: Traditional ensemble methods are learning algorithms that construct a set of unit models and combine the learning results for classification or regression problems. There is no definitive taxonomy of ensemble learning actually. According to the existing works, different people categorize the ensemble learning techniques into different ways. Jain et al. \cite{JDM00} list $18$ different classifier combination methods; Witten and Frank \cite{WFH11} detail four methods of combining multiple models: bagging, boosting, stacking and error-correcting output codes; Bishop \cite{B06} covers BMA, committees, boosting, tree-based models and conditional mixture models; Marsland \cite{M09} covers boosting (AdaBoost and stumping), bagging (including subagging) and the mixture of experts method; whilst Alpaydin \cite{A10} covers seven methods of combining multiple learners: voting, error-correcting output codes, bagging, boosting, mixtures of experts, stacked generalization and cascading. By this context so far, the representative ensemble techniques include \textit{bagging} \cite{B96}, \textit{boosting} \cite{S90, FS97}, \textit{stacking} \cite{WD88}, and \textit{random subspace method} \cite{H98}.

\vspace{-10pt}
\section{Notations and Terminology Definitions}\label{sec:definition}
\vspace{-10pt}

In this section, we will introduce the notations and definitions of several important terminologies, which will be used throughout this paper.

\vspace{-10pt}
\subsection{Notations}
\vspace{-8pt}

In the sequel of this paper, we will use the lower case letters (e.g., $x$) to represent scalars, lower case bold letters (e.g., $\mb{x}$) to denote column vectors, bold-face upper case letters (e.g., $\mb{X}$) to denote matrices, and upper case calligraphic letters (e.g., $\mathcal{X}$) to denote sets. Given a matrix $\mb{X}$, we denote $\mb{X}(i,:)$ and $\mb{X}(:,j)$ as the $i_{th}$ row and $j_{th}$ column of matrix $\mb{X}$ respectively. The ($i_{th}$, $j_{th}$) entry of matrix $\mb{X}$ can be denoted as either $X(i,j)$ or $X_{i,j}$, which will be used interchangeably in this paper. We use $\mb{X}^\top$ and $\mb{x}^\top$ to represent the transpose of matrix $\mb{X}$ and vector $\mb{x}$. We will use $\mathbb{R}^d$ to denote an Euclidean space of dimension $d$.

\vspace{-10pt}
\subsection{Terminology Definitions}
\vspace{-8pt}

\begin{defn} 
(\textbf{Discriminatory Function}): Formally, a function $\sigma(\cdot): \mathbb{R} \to \mathbb{R}$ is said to be \textit{discriminatory} iff for a signed Borel measure $\mu$ defined on $\mathbb{R}$, the fact that the following equation 
\begin{equation}
\int_{\mathbb{R}} \sigma ({w} {x} + \theta) \mathrm{d}\mu({x}) = 0
\end{equation}
holds for all ${w}, \theta \in \mathbb{R}$ implies that measure $\mu = 0$.
\end{defn}

\begin{defn} 
(\textbf{Sigmoidal Function}): A function $\sigma(\cdot): \mathbb{R} \to \mathbb{R}$ is said to be \textit{sigmoidal} if it satifies
\begin{equation}
\lim_{x \to \infty} \sigma(x) = 1 \mbox{, and } \lim_{x \to -\infty} \sigma(x) = 0.
\end{equation}
\end{defn}

\begin{defn} 
(\textbf{Bounded Function}): A function $\sigma(\cdot): \mathbb{R} \to \mathbb{R}$ is said to be \textit{bounded} if there exists a real number $B \in \mathbb{R}$ such that  the following inequality holds
\begin{equation}
\left| \sigma(x) \right| \le B, \forall x \in \mathbb{R}.
\end{equation}
\end{defn}
\vspace{-15pt}
\section{Deep Ensemble Learning Model}\label{sec:method}
\vspace{-10pt}

To simplify the problem setting, we will take the traditional single-label binary classification problem as an example to illustrate the \textit{deep ensemble learning} model in this paper, where the label set $\mathcal{Y} = \{0, 1\}$. With all the data instances in $\mathcal{X}$, we will be able to build a group of unit models, $f^i(\cdot; \mb{w}^i): \mb{x} \to y \in \{0, 1\}, \forall i \in \{1, 2, \cdots, n\}$. 

\vspace{-10pt}
\subsection{Independent Ensemble Learning}\label{subsec:ensemble}
\vspace{-8pt}

Given a true mapping $g(\cdot; \mb{\theta})$ and a set of unit models $\{f^1(\cdot;\mb{w}^1), f^2(\cdot;\mb{w}^2), \cdots, f^n(\cdot;\mb{w}^n)\}$, if we assume these models are all independent and the can all achieve an error rate no greater than $\epsilon$, where $\epsilon < \frac{1}{2}$, we can have
\begin{equation}
P\left(f^i(\mb{x}; \mb{w}^i) \neq g(\mb{x}; \mb{\theta}) \right) \le \epsilon, \forall i \in \{1, 2, \cdots, n\}.
\end{equation}
By combining these independent unit models together with equal weights, i.e., $v^i = \frac{1}{n}, \forall i \in \{1, 2, \cdots, n\}$, and assuming the \textit{ensemble function} $h(z) = sign(z)$, we will have the prediction result of the \textit{deep ensemble learning} model on data instance featured by vector $\mb{x}$ as\begingroup\makeatletter\def\f@size{8}\check@mathfonts
\begin{equation}
F(\mb{x}) = h\left( \sum_{i = 1}^n v^i \cdot f^i(\mb{x}; \mb{w}^i) \right) = sign\left( \frac{1}{n} \cdot \sum_{i = 1}^n f^i(\mb{x}; \mb{w}^i) \right).
\end{equation}\endgroup
Among these $n$ unit models, if half of them predict the results correctly, then the \textit{deep ensemble learning} model will output a correct prediction. Formally, the learning error rate for the \textit{deep ensemble learning} model can be represented as \begingroup\makeatletter\def\f@size{8}\check@mathfonts
\begin{equation}
P\left(F(\mb{x}) \neq g(\mb{x}; \mb{\theta}) \right) \le \sum_{k = 0}^{\lfloor n/2 \rfloor}  \dbinom{n}{k} (1-\epsilon)^k \epsilon^{n - k} \le \exp^{\left(- \frac{n}{2} (1- 2\epsilon)^2 \right)}
\end{equation}\endgroup
In other words, the more unit models we have (i.e., $n$ is larger), the \textit{deep ensemble leanring} model error rate will have a tighter bound, which decreases exponentially with $n$. However, in the real world, the unit models learned with the same training set are usually strongly correlated. In the ensemble step, different unit models should also have different weights. Therefore, the above performance bound can hardly be used in concrete learning applications and analysis.

\vspace{-10pt}
\subsection{General Ensemble Learning Model for Universal Function Approximation}\label{subsec:analysis}
\vspace{-8pt}

The unit models involved in ensemble learning are usually weak learning methods, which can be in different forms. Considering that the binary classification problem is studied as an example in this paper, we can formally represent the unit models as $f(\mb{x}; \mb{w}) = \sigma(\mb{w}^\top \mb{x} + w_0)$. Here, function $\sigma(\cdot): \mathbb{R} \to \mathbb{R}$ projects the computed instance weighted sum score to the binary label space $\{0, 1\}$, which can have different definitions. In the case when $\sigma(\cdot)$ is the \textit{sigmoid function}, the unit model will be the \textit{logistic regression} model, while if $\sigma(\cdot)$ is the hard \textit{sign function}, the unit model will be SVM instead. Besides these unit model representations introduced here, some other unit learning models, e.g., \textit{decision tree}, can also be adopted as indicated in \cite{ZF17}.

The definition of function $\sigma(\cdot)$ may have a significant impact on the learning performance of the \textit{deep ensemble learning} model, and we will demonstrate that \textit{bounded sigmoidal function} will allow the ensemble learning model to achieve a very powerful representation capacity.

\begin{lemma}\label{lemma:1}
Any bounded measurable sigmoidal function $\sigma(\cdot)$ is discriminatory.
\end{lemma}

\begin{proof}
The full proof of the Lemma is available in \cite{C89}. Here, we will provide a brief introduction to the proof as follows. We propose to introduce a function $\sigma \left( \lambda(\mb{w}^\top \mb{x} + w_0) + \varphi \right)$. Since $\sigma(\cdot)$ is \textit{sigmoidal}, we will have the following equations hold\begingroup\makeatletter\def\f@size{8}\check@mathfonts
\begin{equation}
\sigma \left( \lambda(\mb{w}^\top \mb{x} + w_0) + \varphi \right) 
\begin{cases}
\to 1 &\mbox{ for } \mb{w}^\top \mb{x} + w_0 > 0 \mbox{ as } \lambda \to + \infty ,\\
\to 0 &\mbox{ for } \mb{w}^\top \mb{x} + w_0 < 0 \mbox{ as } \lambda \to + \infty ,\\
= \sigma(\varphi) & \mbox{ for } \mb{w}^\top \mb{x} + w_0 = 0 \mbox{ and for } \forall \lambda \in \mathbb{R}.
\end{cases}
\end{equation}\endgroup
Therefore, function $\sigma_{\lambda}(\mb{x}) = \sigma \left( \lambda(\mb{w}^\top \mb{x} + w_0) + \varphi \right)$ will converge pointwisely and boundedly to\begingroup\makeatletter\def\f@size{8}\check@mathfonts
\begin{equation}
\gamma(\mb{x}) = 
\begin{cases}
1& \mbox{ for } \mb{w}^\top \mb{x} + w_0 > 0,\\
0& \mbox{ for } \mb{w}^\top \mb{x} + w_0 < 0,\\
\sigma(\varphi)& \mbox{ for } \mb{w}^\top \mb{x} + w_0 = 0,\\
\end{cases}
\end{equation}\endgroup
Here, we can denote $\boldsymbol{\Pi}_{\mb{w}, w_0}$ as the hyperplane defined by $\{\mb{x} | \mb{w}^\top \mb{x} + w_0 = 0 \}$ and $\mb{H}_{\mb{w}, w_0}$ as the open half-space defined by $\{\mb{x} | \mb{w}^\top \mb{x} + w_0 > 0\}$. According to the Lesbegue Bounnded Convergence Theorem, we have\begingroup\makeatletter\def\f@size{8}\check@mathfonts
\begin{align}
0 &= \int_{\mathbb{R}^d} \sigma_{\lambda}(\mb{x}) \mathrm{d} \mu (\mb{x}) = \int_{\mathbb{R}^d} \gamma(\mb{x}) \mathrm{d} \mu (\mb{x})\\
&= \sigma(\varphi) \mu(\boldsymbol{\Pi}_{\mb{w}, w_0}) + \mu(\mb{H}_{\mb{w}, w_0}).
\end{align}\endgroup
As indicated in \cite{C89}, the measure of all half-planes being $0$ implies that the measure $\mu$ must be $0$ actually, which demonstrates that bounded measurable sigmoidal function $\sigma(\cdot)$ is discriminatory.
\end{proof}

\begin{theo}\label{theo:approximation}
Let $\sigma(\cdot)$ be a bounded measurable sigmoidal function. Then, for any function $g(\cdot; \mb{\theta})$, there exists a deep ensemble function\begingroup\makeatletter\def\f@size{8}\check@mathfonts
\begin{equation}
F(\mb{x}) = h\left( \sum_{i = 1}^n v^i \cdot f^i(\mb{x}; \mb{w}^i) \right) = h\left( \sum_{i = 1}^n v^i \cdot \sigma^i \left( (\mb{w}^i)^\top \mb{x} + w^i_0 \right) \right),
\end{equation}\endgroup
with a finite number $n$ and $\mb{w}^i \in \mathbb{R}^{d}$, $w^i_0 \in \mathbb{R}$ such that for $\forall \mb{x} \in \mathbb{R}^d$ and $\forall \epsilon \in [0, 1]$, we have
\begin{equation}
\left | F(\mb{x}) - g(\mb{x}; \mb{\theta}) \right| < \epsilon.
\end{equation}
\end{theo}

\begin{proof}
According to Lemma~\ref{lemma:1}, given $\sigma(\cdot)$ is a bounded measurable sigmoidal function, we know $\sigma(\cdot)$ should be discriminatory as well. Here, let notation $C(\mathbb{R}^d)$ denote the set of continuous functions based on space $\mathbb{R}^d$ and $\mathcal{S} \subset C(\mathbb{R}^d)$ denote the set of functions of the form $F(\mb{x})$ according to the representation in the theorem. To prove the theorem, we only need to demonstrate that the closure of $\mathcal{S}$ is equal to $C(\mathbb{R}^d)$ actually. 

Here, if we assume the closure of $\mathcal{S}$ (we can denote its closure as $\mathcal{R}$) is not equal to $C(\mathbb{R}^d)$, i.e., $\mathcal{R}$ is merely a subset of $C(\mathbb{R}^d)$. According to the Hahn-Banach theorem \cite{CCN97}, there should exist a bounded linear function on $C(\mathbb{R}^d)$, e.g., $L(\cdot)$, such that $L(\cdot) \neq 0$ and $L(\mathcal{S}) = L(\mathcal{R}) = 0$. According to the Riesz Representation Theorem \cite{R87}, the function $L(\cdot)$ is in the form
\begin{equation}
L(h) = \int_{\mathbb{R}^d} h(\mb{x}) \mathrm{d}\mu(\mb{x}),
\end{equation}
for some measure $\mu$ and $\forall h(\cdot) \in C(\mathbb{R}^d)$. Meanwhile, considering that function $\sigma(\mb{w}^\top \mb{x} + w_0)$ is defined based on $\mathbb{R}^d$ for all $\mb{w}$ and $w_0$, we can have that
\begin{equation}
\int_{\mathbb{R}^d} \sigma(\mb{w}^\top \mb{x} + w_0) \mathrm{d}\mu(\mb{x}) = 0.
\end{equation}
Since function $\sigma(\cdot)$ has been shown to be \textit{discriminatory}, we can obtain that measure $\mu = 0$, which means $L(h) = 0, \forall h(\cdot) \in C(\mathbb{R}^d)$. It contradicts our assumption, hence we can prove that the closure of $\mathcal{S}$ is actually equal to $\mathcal{R}$, i.e., for any function $g(\mb{x}; \mb{\theta})$ defined on space $\mathbb{R}^d$, we can identify a function in the representation of $F(\mb{x})$ to approximate it with any pre-specified error rate.
\end{proof}

\vspace{-10pt}
\section{Performance Analysis of Deep Ensemble Learning Model}\label{sec:analysis}
\vspace{-10pt}

In Theorem~\ref{theo:approximation}, we only indicate that there will exist a certain number $n$ to make the theorem hold. Meanwhile, the specified number $n$ required to achieve the universal approximation is still unknown. In this section, we will mainly focus on analyzing the required unit model number $n$ in the \textit{deep ensemble learning} framework.

\vspace{-10pt}
\subsection{Analysis of Value $n$ with One Single Ensemble Layer}
\vspace{-8pt}

In this part, we will study how many unit models are required to achieve a certain pre-specified accuracy for a general function for a discrete feature space $\mathcal{X} = \{0, 1\}^{d}$ specifically (continuous space can be studied in a similar way). In such a feature space, monomial terms will be the basic representation unit for any functions, and we have the following Lemma holds.

\vspace{-5pt}

\begin{lemma}
Based on the space $\mathcal{X} \times \mathcal{Y}$, for any functions $g(\mb{x} ; \mb{\theta})$, where $\mb{\theta} \in \mathbb{R}^{d}$ and $\mb{x} \in \{0, 1\}^{d}$, they can all be represented as a finite weighted sum of monomials about $\mb{x}$ as follows:\vspace{-3pt}\begingroup\makeatletter\def\f@size{8}\check@mathfonts
\begin{equation}
g(\mb{x} ; \mb{\theta}) = \sum_{n = 0}^{d} \sum_{1 \le i_1 < i_2 < \cdots < i_n \le d} v^{(n)}_{{i_1,i_2, \cdots, i_n}} \cdot  x_{i_1} x_{i_2} \cdots x_{i_n},
\end{equation} \endgroup

\vspace{-3pt}
\noindent where term $v^{(n)}_{{i_1,i_2, \cdots, i_n}} \in \mathbb{R}$ denotes the weight represented by parameters in $\mb{\theta}$. 
\end{lemma}

\vspace{-5pt}

The Lemma can be proved with the \textit{Taylor expansion theorem}, and we will not provide the detailed proof here due to the limited space. According to the Lemma, next our analysis will be mainly based on a monomial term $\prod_{i = 1}^d x_i$ specifically.

\vspace{-8pt}
\begin{theo}\label{theo:shallow}
Given a monomial function $\prod_{i = 1}^d x_i$, for any desired accuracy $\epsilon$, it can be effectively approximated with the 1-layer ensemble learning model defined above with $2^d$ unit models.
\end{theo}

\vspace{-8pt}

\begin{proof}
The proof to the theorem has two main parts: (1) $2^d$ unit models are sufficient, and (2) $2^d$ unit models are necessary.

\noindent \textbf{Part 1}: We propose to prove the ``sufficiency'' by constructing a feasible solution to the ensemble learning model that can approximate the monomial term $\prod_{i = 1}^d x_i$. Here, the main objective is to prove that the following equation can hold \vspace{-5pt}\begingroup\makeatletter\def\f@size{8}\check@mathfonts
\begin{equation}\label{equ:expansion}
 F(\mb{x}) = \sum_{i = 1}^n v^i \cdot \sigma^i \left( (\mb{w}^i)^\top \mb{x} + w^i_0 \right)  \approx  \prod_{i = 1}^d x_i.
\end{equation}\endgroup
where, the operator $\approx$ denotes that the two sides have identical Taylor expansion representations. In the proof, we will assume function $\sigma(\cdot)$ is the sigmoid function. According to the Taylor Expansion Theorem, at point $\mb{x}_0 = \mb{0}$, function $F(\mb{x})$ can be approximately represented as \begingroup\makeatletter\def\f@size{8}\check@mathfonts
\begin{align}
F(\mb{x}) &= \sum_{k = 0}^n \sum_{j_1 = 1}^d \sum_{j_2 = 1}^d \cdots \sum_{j_k = 1}^d \frac{ \partial^k F(\mb{x}_0) }{ \partial x_{j_1} \partial x_{j_2} \cdots \partial x_{j_k} }   x_{j_1} x_{j_2} \cdots x_{j_k} + R_{n+1}(\mb{x})\\
&= F(\mb{0})  \sum_{i = 1}^n v^i \cdot \sum_{k = 0}^n \left( \sum_{j = 1}^d w^i_{j}  x_{j} \right)^k + R_{n+1}(\mb{x}).
\end{align}
Therefore, equation~\ref{equ:expansion} holds iff the following two equations can both hold
\begin{equation}\label{equ:cases}
\begin{cases}
\mbox{E1: }& F(\mb{0})  \sum_{i = 1}^n v^i \cdot \left( \sum_{j = 1}^d w^i_{j}  x_{j} \right)^n = \prod_{i = 1}^n x_i\\
\mbox{E2: }&F(\mb{0})  \sum_{i = 1}^n v^i \cdot \left( \sum_{j = 1}^d w^i_{j}  x_{j} \right)^k = 0, \forall k \in \{0, 1, 2, \cdots, n - 1\}.
\end{cases}
\end{equation}\endgroup
The high-order reminder is omitted since it will be of a extremely small value, removal of which will not have a large impact on the approximation results.

We propose to construct a feasible solution to the above equations as indicated in \cite{lin2016cheap}. Given a set $\{1, 2, \cdots, d\}$, we can enumerate all its unique subsets as $\mathcal{S}_1, \cdots, \mathcal{S}_n$, where $n = 2^d$. Here, we assign the variable $w^i_j$ with the results from function $s_j(\mathcal{S}_i)$, where $s_j(\mathcal{S}_i) = -1$ if $j \in \mathcal{S}_i$; otherwise $s_j(\mathcal{S}_i) = +1$. Furthermore, we can assign $v^i$ with value in $\{-1, +1\}$ according to the following equation\begingroup\makeatletter\def\f@size{8}\check@mathfonts
\begin{equation}
v^i = \frac{1}{ 2^d d! F(\mb{0})  } \prod_{j = 1}^d w_j^i = \frac{(-1)^{|\mathcal{S}_i|}}{2^d d! F(\mb{0})  }.
\end{equation}\endgroup
Next, we only need to demonstrate that variables $\{v^i\}_{i = 1}^n$ and $\{w_j^i\}_{i = 1, j = 1}^{n, d}$ can make the equations in Equ~\ref{equ:cases} both hold. We introduce a function $b(\mb{x}) = x_1^{c_1} x_2^{c_2} \cdots x_d^{c_d}$, where $\sum_{i = 1}^n c_i = c \le d$. First of all, in the case when $b(\mb{x}) \neq \prod_{i = 1}^d x_i$. In other words, we have the exponent $c_l$ of certain term $x_l$ in function $b(\mb{x})$ to be zero, i.e., $x_l$ is not involved in $b(\mb{x})$. According to the previous variable definition, we have term \begingroup\makeatletter\def\f@size{8}\check@mathfonts
\begin{align}
&F(\mb{0})  \sum_{i = 1}^n v^i \cdot \left( \sum_{j = 1}^d w^i_{j}  x_{j} \right)^c = F(\mb{0})  \sum_{\mathcal{S}_i} \frac{(-1)^{|\mathcal{S}_i|}}{2^d d! F(\mb{0})  }  \cdot \left( \sum_{j = 1}^d  s_j(\mathcal{S}_i)   x_{j} \right)^c\\
&= \sum_{\mathcal{S}_i, l \notin \mathcal{S}_i} \left[   \frac{(-1)^{|\mathcal{S}_i|}}{2^d d!  }  \left( \sum_{j = 1}^d  s_j(\mathcal{S}_i)   x_{j} \right)^c +   \frac{(-1)^{|\mathcal{S}_i \cup \{l\}|}}{2^d d!  }  \left( \sum_{j = 1}^d  s_j(\mathcal{S}_i \cup \{l\})   x_{j} \right)^c       \right],\\
&= \sum_{\mathcal{S}_i, l \notin \mathcal{S}_i} \frac{(-1)^{|\mathcal{S}_i|}}{2^d d!  }  \left[ \left( \sum_{j = 1}^d  s_j(\mathcal{S}_i)   x_{j} \right)^c - \left( \sum_{j = 1}^d  s_j(\mathcal{S}_i \cup \{l\})   x_{j} \right)^c  \right].
\end{align}\endgroup
In the case when $c_l = 0$, we know that the coefficient of term $b(\mb{x})$ in term $\left( \sum_{j = 1}^d  s_j(\mathcal{S}_i)   x_{j} \right)^c$ actually equals to that in $ \left( \sum_{j = 1}^d  s_j(\mathcal{S}_i \cup \{l\})   x_{j} \right)^c$. In other words, we can demonstrate the coefficient of $b(\mb{x})$ in function $F(\mb{0})  \sum_{i = 1}^n v^i \cdot \left( \sum_{j = 1}^d w^i_{j}  x_{j} \right)^c$ will be equal to $0$, i.e., $E2$ in Equ~\ref{equ:cases}.

On the other hand, in the case when function $b(\mb{x}) = \prod_{i = 1}^d x_i$, we have the corresponding coefficient of $b(\mb{x})$ in term $\left( \sum_{j = 1}^d w^i_{j}  x_{j} \right)^d$ to be $d! \prod_{j = 1}^d w^i_j = (-1)^{|\mathcal{S}_i|} d!$. Therefore, we have the coefficient of term $b(\mb{x}) = \prod_{i = 1}^d x_i$ in function $F(\mb{x})$ to be\begingroup\makeatletter\def\f@size{8}\check@mathfonts
\begin{align}
F(\mb{0})  \sum_{i = 1}^n v^i \cdot (-1)^{|\mathcal{S}_j|} d! = F(\mb{0})  \sum_{i = 1}^n \frac{(-1)^{|\mathcal{S}_i|}}{2^d d! F(\mb{0})} (-1)^{|\mathcal{S}_i|} d! = 1. \mbox{ i.e., $E1$ in Equ~\ref{equ:cases}.}
\end{align}\endgroup

\vspace{-10pt}

\noindent \textbf{Part 2}: This part can be proved by constructing a matrix $\mb{A} \in \mathbb{R}^{2^d \times n}$, where entry $A(\mathcal{S}, j) = \prod_{i \in \mathcal{S}} w^j_i$. Each row of matrix $\mb{A}$ corresponds to a subset $\mathcal{S}$ of $\{1, 2, \cdots, d\}$ and each column represents the subsets $\{\mathcal{S}_1, \mathcal{S}_2, \cdots, \mathcal{S}_n\}$ or certain order. Here, we want to demonstrate that $2^d$ unit models are necessary, which means we need to demonstrate that $n \ge 2^d$ (i.e., matrix $\mb{A}$ is full rank).

We will prove it by contradiction, and let's assume matrix $\mb{A}$ is not of a full rank, and there exists some vector $\mb{a} \in \mathbb{R}^{2^d}$ and $\mb{a} \neq \mb{0}$ that makes $\mb{a}^\top \mb{A} = \mb{0}$. We propose to construct another vector $\mb{b} \in \mathbb{R}^{n}$ with entry $b_i = v^i \left( \sum_{k = 1}^d w_k^i x_k \right)^{d-s}$, where $s$ denotes the maximal cardinality of subset $\mathcal{S}$ corresponding to the rows of matrix $\mb{A}$. Therefore, we will get  \begingroup\makeatletter\def\f@size{8}\check@mathfonts
\begin{align}
&{0} = \mb{a}^\top \mb{A} \mb{b} = \sum_{i = 1}^r a_i \sum_{j = 1}^m v^j \left( \sum_{k = 1}^d w_k^j x_k \right)^{d-s} \prod_{h \in \mathcal{S}_i} w_h^j \\
&= \hspace{-7pt} \sum_{i, |\mathcal{S}_i| = s} \hspace{-5pt} a_i \sum_{j = 1}^m v^j \left( \sum_{k = 1}^d w_k^j x_k \right)^{d-|\mathcal{S}_i|} \hspace{-15pt} \prod_{h \in \mathcal{S}_i} w_h^j + \hspace{-7pt} \sum_{i, |\mathcal{S}_i| < s} \hspace{-8pt} a_i \sum_{j = 1}^m v^j \left( \sum_{k = 1}^d w_k^j x_k \right)^{ \hspace{-5pt} (d + |\mathcal{S}_i| -s ) - |\mathcal{S}_i|} \hspace{-15pt} \prod_{h \in \mathcal{S}_i} w_h^j.
\end{align}\endgroup
Here, we notice that by taking the derivatives of E2 in Equ~\ref{equ:cases} regarding $\{x_j\}_{j \in \mathcal{S}_i}$, we can prove the second term in the above equation equals to $0$ actually. Meanwhile, by taking the derivative of E1 in Equ~\ref{equ:cases} regarding $\{x_j\}_{j \in \mathcal{S}_i}$, we can get above equation can be simplified as follows\begingroup\makeatletter\def\f@size{8}\check@mathfonts
\begin{align}
{0} = \sum_{i, |\mathcal{S}_i| = s} \frac{a_i  \left| d - \mathcal{S}_i \right|!}{d! F(\mb{0})} \prod_{h \notin \mathcal{S}_i} x_h.
\end{align}\endgroup
In other words, the monomials denoted by terms $\{ \prod_{h \notin \mathcal{S}_i} x_h \}$ has shown to be linearly dependent, which contradicts the fact that distinct monomials are linearly independent. Therefore, our assumption cannot hold actually, and we can demonstrate that $n \ge 2^d$. 
\end{proof}

\vspace{-12pt}
\subsection{Analysis of Value $n$ with Deep Ensemble Learning Model with Multiple Layers}
\vspace{-8pt}

In the case when the \textit{deep ensemble learning} model involves multiple layers in the ensemble step, we will analyze its corresponding benefits in this part as follows. Formally, let $F^{i,k}(\mb{x})$ (where $k \in \{1, 2, \cdots, K\}$) denote the first ensemble layer of length $K$, which accepts input from unit models $\{f^{1,k}_1(\cdot; \mb{w}^{1,k}_1), f^{1,k}_2(\cdot; \mb{w}^{1,k}_2), \cdots, f^{1,k}_n(\cdot; \mb{w}^{1,k}_n)\}$. The ensemble results from $[F^{1,k}(\mb{x})]_{k=1}^K$ will be fed as the input for the second ensemble layer $F^{2,k}( [F^{1,k}(\mb{x})]_{k=1}^K )$ (where $k \in \{1, 2, \cdots, K\}$) and so forth until the $N_{th}$ layer. 

\vspace{-10pt}

\begin{theo}
Given a monomial term $\prod_{i = 1}^d x_i$ and any pre-specified accuracy $\epsilon$, it can be represented as the \textit{deep ensemble learning} model with the minimum $O(d)$ unit models and $\lceil \log d \rceil + 1$ ensemble layers.
\end{theo}

\vspace{-10pt}

\begin{proof}
We propose to prove the above theorem by constructing a deep ensemble learning model. Formally, term $\prod_{i = 1}^d x_i$ can be precisely represented as an recursive product of pairwise features, e.g., $\prod_{i = 1}^d x_i = (x_1 \cdot (x_2 \cdots ( ... (x_{d-1} \cdot x_d) ) )$. Depending on how we construct the product function term, the computation process can be represented with binary trees indifferent shapes, where $\{x_1, x_2, \cdots, x_d\}$ denote the leaf nodes and the \textit{intermediate result} as the internal nodes.

Meanwhile, for a binary tree with $d$ leaf nodes, the number of internal nodes involved in it will be $d-1$, regardless of the tree shape. For each internal node in the binary tree, it will accept two inputs from either leaf node or internal node, and compute their product as the output. Meanwhile, according to Theorem~\ref{theo:shallow}, the computation of each internal node can be approximated to any accuracy with an ensemble learning model involving $2^2$ unit models.

The specific number of layers involved is highly dependent on the shape of binary tree about the deep ensemble learning model. For instance, for the function decomposed in the shape of a full binary tree (just like function $\prod_{i = 1}^d x_i = (x_1 \cdot (x_2 \cdots ( ... (x_{d-1} \cdot x_d) ) )$), its level (i.e., model ensemble layer) will be $d-1$, which is the maximum hight. Meanwhile, for the function decomposed in the shape of a complete binary tree, its number of levels (i.e., model ensemble layer) will be $\lceil \log d \rceil + 1$, which is the minimum height.
\end{proof}

\vspace{-10pt}
In other words, with a deep ensemble layer, we will be able to greatly reduce the number of required unit models. Considering that there exists various of variables to be learned in each base learners, reducing the unit model number will greatly lower down the number of variables at the same time. As indicated by the Vapnik-Chervonenkis theory, the number of variables will greatly limit the generalizability of the learning model and can lead to overfitting easily especially when the training instance is not sufficient. Therefore, the deep ensemble layer will also behave much more robustly compared with the shallow ensemble model.

\vspace{-10pt}
\section{Numerical Experiments}\label{sec:experiment}
\vspace{-10pt}

To test the effectiveness of the proposed algorithm, extensive numerical experiments have been done on several real-world benchmark datasets.

\begin{table*}[t]
\caption{Learning Results on MNIST Dataset.}
\scriptsize
\centering
\setlength{\tabcolsep}{4pt}
\begin{tabular}{l  cccccccc  }
\toprule

methods &\textbf{Accuracy} &\textbf{F1} &\textbf{Precision} &\textbf{Recall} &\textbf{MSE} &\textbf{MAE} &{$R^2$} &\textbf{Avg. Rank}  \\
\midrule 

{Deep Ensemble}	
&\textbf{0.992} {\tiny {\color{blue} (1)}}
&\textbf{0.992} {\tiny {\color{blue} (1)}}
&\textbf{0.992} {\tiny {\color{blue} (1)}}
&\textbf{0.992} {\tiny {\color{blue} (1)}}
&\textbf{0.174} {\tiny {\color{blue} (2)}}
&\textbf{0.033} {\tiny {\color{blue} (1)}}
&\textbf{0.979} {\tiny {\color{blue} (2)}}
&{\color{blue} (1)}\\

{Majority Vote}
&\textbf{0.986} {\tiny {\color{blue} (3)}}
&\textbf{0.986} {\tiny {\color{blue} (3)}}
&\textbf{0.986} {\tiny {\color{blue} (3)}}
&\textbf{0.986} {\tiny {\color{blue} (3)}}
&\textbf{0.157} {\tiny {\color{blue} (1)}}
&\textbf{0.038} {\tiny {\color{blue} (2)}}
&\textbf{0.981} {\tiny {\color{blue} (1)}}
&{\color{blue} (2)}\\

\midrule 

{Adaboost} \cite{S99}	
&0.713 {\tiny {\color{blue} (10)}}
&0.705 {\tiny {\color{blue} (10)}}
&0.714 {\tiny {\color{blue} (10)}}
&0.713 {\tiny {\color{blue} (10)}}
&4.804 {\tiny {\color{blue} (10)}}
&1.054 {\tiny {\color{blue} (10)}}
&0.427 {\tiny {\color{blue} (10)}}
&{\color{blue} (10)}\\

{CNN} \cite{KSH12}	
&\textbf{0.988} {\tiny {\color{blue} (2)}}
&\textbf{0.988} {\tiny {\color{blue} (2)}}
&\textbf{0.988} {\tiny {\color{blue} (2)}}
&\textbf{0.988} {\tiny {\color{blue} (2)}}
&\textbf{0.236} {\tiny {\color{blue} (3)}}
&\textbf{0.047} {\tiny {\color{blue} (3)}}
&\textbf{0.972} {\tiny {\color{blue} (3)}}
&{\color{blue} (3)}\\

{Decision Tree} \cite{Q86}	
&0.876 {\tiny {\color{blue} (9)}}
&0.876 {\tiny {\color{blue} (9)}}
&0.876 {\tiny {\color{blue} (9)}}
&0.876 {\tiny {\color{blue} (9)}}
&2.069 {\tiny {\color{blue} (9)}}
&0.437 {\tiny {\color{blue} (9)}}
&0.753 {\tiny {\color{blue} (9)}}
&{\color{blue} (9)}\\

{k-Nearest Neighbor} \cite{BGRS99}	
&0.970 {\tiny {\color{blue} (5)}}
&0.970 {\tiny {\color{blue} (5)}}
&0.970 {\tiny {\color{blue} (5)}}
&0.970 {\tiny {\color{blue} (5)}}
&0.586 {\tiny {\color{blue} (5)}}
&0.118 {\tiny {\color{blue} (5)}}
&0.930 {\tiny {\color{blue} (5)}}
&{\color{blue} (5)}\\

{Logistic Regression} \cite{PLI02}	
&0.920 {\tiny {\color{blue} (7)}}
&0.920 {\tiny {\color{blue} (7)}}
&0.920 {\tiny {\color{blue} (7)}}
&0.920 {\tiny {\color{blue} (7)}}
&1.405 {\tiny {\color{blue} (7)}}
&0.292 {\tiny {\color{blue} (7)}}
&0.832 {\tiny {\color{blue} (7)}}
&{\color{blue} (7)}\\

{Multi-layer Perceptron} \cite{B95}	
&0.975 {\tiny {\color{blue} (4)}}
&0.975 {\tiny {\color{blue} (4)}}
&0.975 {\tiny {\color{blue} (4)}}
&0.975 {\tiny {\color{blue} (4)}}
&0.482 {\tiny {\color{blue} (4)}}
&0.094 {\tiny {\color{blue} (4)}}
&0.942 {\tiny {\color{blue} (4)}}
&{\color{blue} (4)}\\

{Naive Bayesian} \cite{DP97}	
&0.549 {\tiny {\color{blue} (11)}}
&0.511 {\tiny {\color{blue} (11)}}
&0.693 {\tiny {\color{blue} (11)}}
&0.549 {\tiny {\color{blue} (11)}}
&7.677 {\tiny {\color{blue} (11)}}
&1.676 {\tiny {\color{blue} (11)}}
&0.084 {\tiny {\color{blue} (11)}}
&{\color{blue} (11)}\\

{Random Forest} \cite{B01}	
&0.945 {\tiny {\color{blue} (6)}}
&0.945 {\tiny {\color{blue} (6)}}
&0.945 {\tiny {\color{blue} (6)}}
&0.945 {\tiny {\color{blue} (6)}}
&0.966 {\tiny {\color{blue} (6)}}
&0.203 {\tiny {\color{blue} (6)}}
&0.885 {\tiny {\color{blue} (6)}}
&{\color{blue} (6)}\\

{Support Vector Machine} \cite{CV95}	
&0.917 {\tiny {\color{blue} (8)}}
&0.917 {\tiny {\color{blue} (8)}}
&0.917 {\tiny {\color{blue} (8)}}
&0.917 {\tiny {\color{blue} (8)}}
&1.423 {\tiny {\color{blue} (8)}}
&0.299 {\tiny {\color{blue} (8)}}
&0.830 {\tiny {\color{blue} (8)}}
&{\color{blue} (8)}\\

\bottomrule
\end{tabular}\label{tab:MNIST}
\end{table*}

\begin{table*}[t]
\caption{Learning Results on ORL Dataset.}
\scriptsize
\centering
\setlength{\tabcolsep}{4pt}
\begin{tabular}{l  cccccccc  }
\toprule

methods &\textbf{Accuracy} &\textbf{F1} &\textbf{Precision} &\textbf{Recall} &\textbf{MSE} &\textbf{MAE} &{$R^2$} &\textbf{Avg. Rank}  \\
\midrule 

{Deep Ensemble}	
&\textbf{0.960} {\tiny {\color{blue} (1)}}
&\textbf{0.960} {\tiny {\color{blue} (1)}}
&\textbf{0.970} {\tiny {\color{blue} (1)}}
&\textbf{0.960} {\tiny {\color{blue} (1)}}
&\textbf{25.270} {\tiny {\color{blue} (1)}}
&\textbf{0.870} {\tiny {\color{blue} (1)}}
&\textbf{0.810} {\tiny {\color{blue} (1)}}
&{\color{blue} (1)}\\

{Majority Vote}
&\textbf{0.955} {\tiny {\color{blue} (2)}}
&\textbf{0.955} {\tiny {\color{blue} (2)}}
&\textbf{0.955} {\tiny {\color{blue} (2)}}
&\textbf{0.955} {\tiny {\color{blue} (2)}}
&29.790 {\tiny {\color{blue} (4)}}
&\textbf{0.960} {\tiny {\color{blue} (3)}}
&0.776 {\tiny {\color{blue} (4)}}
&{\color{blue} (3)}\\

\midrule 

{Adaboost} \cite{S99}	
&0.130 {\tiny {\color{blue} (10)}}
&0.112 {\tiny {\color{blue} (10)}}
&0.151 {\tiny {\color{blue} (10)}}
&0.130 {\tiny {\color{blue} (10)}}
&206.900 {\tiny {\color{blue} (10)}}
&11.100 {\tiny {\color{blue} (10)}}
&-0.553 {\tiny {\color{blue} (10)}}
&{\color{blue} (10)}\\

{CNN} \cite{KSH12}	
&\textbf{0.945} {\tiny {\color{blue} (3)}}
&\textbf{0.945} {\tiny {\color{blue} (3)}}
&\textbf{0.953} {\tiny {\color{blue} (3)}}
&\textbf{0.945} {\tiny {\color{blue} (3)}}
&32.195 {\tiny {\color{blue} (5)}}
&1.115 {\tiny {\color{blue} (5)}}
&0.758 {\tiny {\color{blue} (5)}}
&{\color{blue} (4)}\\

{Decision Tree} \cite{Q86}	
&0.505 {\tiny {\color{blue} (9)}}
&0.498 {\tiny {\color{blue} (9)}}
&0.546 {\tiny {\color{blue} (9)}}
&0.505 {\tiny {\color{blue} (9)}}
&142.250 {\tiny {\color{blue} (9)}}
&7.160 {\tiny {\color{blue} (9)}}
&-0.068 {\tiny {\color{blue} (9)}}
&{\color{blue} (9)}\\

{k-Nearest Neighbor} \cite{BGRS99}	
&0.905 {\tiny {\color{blue} (6)}}
&0.901 {\tiny {\color{blue} (6)}}
&0.930 {\tiny {\color{blue} (6)}}
&0.905 {\tiny {\color{blue} (6)}}
&34.205 {\tiny {\color{blue} (6)}}
&1.455 {\tiny {\color{blue} (6)}}
&0.743 {\tiny {\color{blue} (6)}}
&{\color{blue} (6)}\\

{Logistic Regression} \cite{PLI02}	
&0.940 {\tiny {\color{blue} (5)}}
&0.940 {\tiny {\color{blue} (5)}}
&0.949 {\tiny {\color{blue} (5)}}
&0.940 {\tiny {\color{blue} (5)}}
&\textbf{28.020} {\tiny {\color{blue} (3)}}
&1.060 {\tiny {\color{blue} (4)}}
&\textbf{0.790} {\tiny {\color{blue} (3)}}
&{\color{blue} (5)}\\

{Multi-layer Perceptron} \cite{B95}	
&0.050 {\tiny {\color{blue} (11)}}
&0.012 {\tiny {\color{blue} (11)}}
&0.008 {\tiny {\color{blue} (11)}}
&0.050 {\tiny {\color{blue} (11)}}
&262.325 {\tiny {\color{blue} (11)}}
&13.415 {\tiny {\color{blue} (11)}}
&-0.969 {\tiny {\color{blue} (11)}}
&{\color{blue} (11)}\\

{Naive Bayesian} \cite{DP97}	
&0.795 {\tiny {\color{blue} (7)}}
&0.794 {\tiny {\color{blue} (7)}}
&0.882 {\tiny {\color{blue} (7)}}
&0.795 {\tiny {\color{blue} (7)}}
&69.750 {\tiny {\color{blue} (7)}}
&3.170 {\tiny {\color{blue} (7)}}
&0.477 {\tiny {\color{blue} (7)}}
&{\color{blue} (7)}\\

{Random Forest} \cite{B01}	
&0.730 {\tiny {\color{blue} (8)}}
&0.718 {\tiny {\color{blue} (8)}}
&0.768 {\tiny {\color{blue} (8)}}
&0.730 {\tiny {\color{blue} (8)}}
&72.490 {\tiny {\color{blue} (8)}}
&3.730 {\tiny {\color{blue} (8)}}
&0.456 {\tiny {\color{blue} (8)}}
&{\color{blue} (8)}\\

{Support Vector Machine} \cite{CV95}	
&\textbf{0.945} {\tiny {\color{blue} (3)}}
&\textbf{0.945} {\tiny {\color{blue} (3)}}
&\textbf{0.953} {\tiny {\color{blue} (3)}}
&\textbf{0.945} {\tiny {\color{blue} (3)}}
&\textbf{25.600} {\tiny {\color{blue} (2)}}
&\textbf{0.950} {\tiny {\color{blue} (2)}}
&\textbf{0.808} {\tiny {\color{blue} (2)}}
&{\color{blue} (2)}\\

\bottomrule
\end{tabular}\label{tab:ORL}
\end{table*}

\vspace{-10pt}
\subsection{Experimental Settings}
\vspace{-8pt}

Three different types of datasets are used in the experiments, including MNIST\footnote{http://yann.lecun.com/exdb/mnist/}, ORL\footnote{http://www.cl.cam.ac.uk/research/dtg/attarchive/facedatabase.html} and IMDB sentiment\footnote{http://ai.stanford.edu/~amaas/data/sentiment/} datasets. Both MNIST and ORL are image datasets, and IMDB sentiment is a text dataset. The MNIST dataset covers $70,000$ $28 \times 28$ images of handwritten digits, where $60,000$ are covered in the training set and $10,000$ are involved in the test set. The labels of the images are selected from set $\{0, 1, \cdots, 9\}$. The ORL dataset covers the $10$ face images of $40$ entities respectively, where the resolution of each images is $112 \times 92$. In the experiments, the images of the same entities partitioned into 2 subsets according to ratio $5:5$ respectively, where half is used as the training set and the other half is used as the testing set. The IMDB sentiment dataset covers the movie reviews crawled from IMDB, and the training set and testing set both contain $25,000$ reviews. These movie reviews are represented as the bag-of-word feature vectors, with labels denoting the scores in range $\{1, 2, \cdots, 10\}$. 

In the experiments various baseline methods are compared against the \textit{deep ensemble learning} model proposed in this paper. Formally, the baseline methods compared in the experiments include: \textit{Convolutional Neural Network} \cite{KSH12}, \textit{Support Vector Machine} (with linear kernel) \cite{CV95}, Logistic Regression \cite{PLI02}, k-Nearest Neighbor \cite{BGRS99}, Naive Bayesian \cite{DP97}, Decision Tree \cite{Q86}, Multi-layer Perceptron \cite{B95}, Random Forest \cite{B01}, and Adaboost \cite{S99}. According to the performance of these baseline methods, the best method will be selected, another $49$ copies of which will be initiated and learned. Based on the $50$ different copies of learning results by the best methods, the \textit{deep ensemble learning} will further combine them together with a deep ensemble layer. To show the advantages of \textit{deep ensemble learning model}, we will also compare it with the \textit{majority vote} ensemble strategy in the experiments. The experimental results will be evaluated by the well used metrics: F1, Precision, Recall, Accuracy, MSE, MAE and $R^2$ respectively. 

\vspace{-10pt}
\subsection{Experimental Results}
\vspace{-8pt}

The experimental results achieved on these $3$ different datasets are provided in Tables~\ref{tab:MNIST}-\ref{tab:IMDB}, where the best results for each metric is in a bolded font and the blue numbers in the brackets denote the relative rank of the methods. The average rank for each method based on these $4$ metrics is also provided in the right column of the tables.

\begin{table*}[t]
\caption{Learning Results on IMDB Sentiment Dataset (Label Space $\{1, 2, \cdots, 10\}$).}
\scriptsize
\centering
\setlength{\tabcolsep}{4pt}
\begin{tabular}{l  cccccccc  }
\toprule

methods &\textbf{Accuracy} &\textbf{F1} &\textbf{Precision} &\textbf{Recall} &\textbf{MSE} &\textbf{MAE} &{$R^2$} &\textbf{Avg. Rank}  \\
\midrule 

{Deep Ensemble}	
&\textbf{0.383} {\tiny {\color{blue} (1)}}
&\textbf{0.368} {\tiny {\color{blue} (1)}}
&\textbf{0.350} {\tiny {\color{blue} (1)}}
&\textbf{0.387} {\tiny {\color{blue} (1)}}
&\textbf{6.314} {\tiny {\color{blue} (1)}}
&\textbf{1.250} {\tiny {\color{blue} (1)}}
&\textbf{0.452} {\tiny {\color{blue} (1)}}
&{\color{blue} (1)}\\

{Majority Vote}
&\textbf{0.367}  {\tiny {\color{blue} (2)}}
&\textbf{0.354}  {\tiny {\color{blue} (2)}}
&\textbf{0.345}  {\tiny {\color{blue} (3)}}
&\textbf{0.367}  {\tiny {\color{blue} (2)}}
&\textbf{7.504}  {\tiny {\color{blue} (3)}}
&\textbf{1.653}  {\tiny {\color{blue} (3)}}
&\textbf{0.385}  {\tiny {\color{blue} (3)}}
&{\color{blue} (2)}\\

\midrule 

{Adaboost} \cite{S99}	
&0.361 {\tiny {\color{blue} (5)}}
&0.285 {\tiny {\color{blue} (6)}}
&0.293 {\tiny {\color{blue} (6)}}
&0.361 {\tiny {\color{blue} (5)}}
&10.766 {\tiny {\color{blue} (6)}}
&2.078 {\tiny {\color{blue} (6)}}
&0.117 {\tiny {\color{blue} (6)}}
&{\color{blue} (6)}\\

{Decision Tree} \cite{Q86}	
&0.252 {\tiny {\color{blue} (8)}}
&0.249 {\tiny {\color{blue} (7)}}
&0.247 {\tiny {\color{blue} (7)}}
&0.252 {\tiny {\color{blue} (8)}}
&14.431 {\tiny {\color{blue} (7)}}
&2.680 {\tiny {\color{blue} (7)}}
&-0.184 {\tiny {\color{blue} (7)}}
&{\color{blue} (7)}\\

{Logistic Regression} \cite{PLI02}	
&\textbf{0.365} {\tiny {\color{blue} (3)}}
&\textbf{0.353} {\tiny {\color{blue} (3)}}
&0.344 {\tiny {\color{blue} (4)}}
&\textbf{0.365} {\tiny {\color{blue} (3)}}
&7.534 {\tiny {\color{blue} (4)}}
&1.700 {\tiny {\color{blue} (4)}}
&0.382 {\tiny {\color{blue} (4)}}
&{\color{blue} (4)}\\

{Multi-layer Perceptron} \cite{B95}	
&0.363 {\tiny {\color{blue} (4)}}
&0.352 {\tiny {\color{blue} (4)}}
&\textbf{0.349} {\tiny {\color{blue} (2)}}
&0.363 {\tiny {\color{blue} (4)}}
&\textbf{6.767} {\tiny {\color{blue} (2)}}
&\textbf{1.627} {\tiny {\color{blue} (2)}}
&\textbf{0.445} {\tiny {\color{blue} (2)}}
&{\color{blue} (3)}\\

{Naive Bayesian} \cite{DP97}	
&0.169 {\tiny {\color{blue} (9)}}
&0.173 {\tiny {\color{blue} (9)}}
&0.188 {\tiny {\color{blue} (9)}}
&0.169 {\tiny {\color{blue} (9)}}
&18.016 {\tiny {\color{blue} (9)}}
&3.204 {\tiny {\color{blue} (9)}}
&-0.478 {\tiny {\color{blue} (9)}}
&{\color{blue} (9)}\\

{Random Forest} \cite{B01}	
&0.284 {\tiny {\color{blue} (7)}}
&0.235 {\tiny {\color{blue} (8)}}
&0.234 {\tiny {\color{blue} (8)}}
&0.284 {\tiny {\color{blue} (7)}}
&17.490 {\tiny {\color{blue} (8)}}
&2.904 {\tiny {\color{blue} (8)}}
&-0.435 {\tiny {\color{blue} (8)}}
&{\color{blue} (8)}\\

{Support Vector Machine} \cite{CV95}	
&0.332 {\tiny {\color{blue} (6)}}
&0.328 {\tiny {\color{blue} (5)}}
&0.324 {\tiny {\color{blue} (5)}}
&0.332 {\tiny {\color{blue} (6)}}
&8.823 {\tiny {\color{blue} (5)}}
&1.901 {\tiny {\color{blue} (5)}}
&0.276 {\tiny {\color{blue} (5)}}
&{\color{blue} (5)}\\

\bottomrule
\end{tabular}\label{tab:IMDB}
\end{table*}

As indicated in Tables~\ref{tab:MNIST} and \ref{tab:ORL}, among the unit baseline methods, CNN can achieve the best performance based on the MNIST and ORL datasets, which is selected as the unit model to ensemble for both MNIST and ORL datasets. For these two ensemble methods, Majority Vote actually achieves very close performance as the unit model. For instance, the Accuracy obtained by Majority Vote on MNIST is $0.986$, which is very close to that obtained by CNN; the Accuracy score obtained by Majority Vote on ORL is $0.955$, which is slightly higher than CNN. On the other hand, the \textit{deep ensemble} method can greatly improve the learning results, which increase the Accuracy of CNN on MNIST and ORL by $0.4\%$ and $1.5\%$ respectively.

Besides the image datasets, for the IMDB textual review dataset, among the baseline methods, Logistic Regression can perform the best among the unit models according to the Accuracy, F1 and Recall metrics according to Table~\ref{tab:IMDB}. Meanwhile, for the remaining metrics, the Multi-Layer Perceptron (MLP) can perform the best. In the experiments, MLP is selected as the unit model for ensemble. According to the results, the \textit{deep ensemble} strategy can improve the learning results compared against MLP by $5.5\%$ for Accuracy and $6.7\%$ for MSE.
\vspace{-10pt}
\section{Conclusion}\label{sec:conclusion}
\vspace{-10pt}

In this paper, we have introduced a novel \textit{deep ensemble learning} model based on deep learning and ensemble learning. In the case when the unit models meet certain properties, i.e., \textit{bounded}, \textit{measurable} and \textit{sigmoidal}, the \textit{deep ensemble learning} model is capable to achieve universal approximation to any functions. Furthermore, according to the analysis provided in the paper, for the \textit{deep ensemble learning} model with one single layer, the number of required unit model to achieve universal approximation is exponential to the input feature dimension $d$, i.e., $O(2^d)$. Meanwhile, as the ensemble layer goes deeper to $\lceil \log d \rceil + 1$, the number of required unit models will decrease dramatically to $O(d)$

\bibliographystyle{abbrv}
\bibliography{reference}

\end{document}